\documentclass[12pt,reqno]{amsart}
\usepackage{amsfonts}
\usepackage{amsmath}
\usepackage{amsthm}
\usepackage{mathrsfs}
\usepackage{amssymb}
\usepackage{stmaryrd}
\usepackage{graphicx}
\usepackage{epstopdf}
\usepackage{textcomp}
\usepackage[margin=1in]{geometry}

\newtheorem{theorem}{Theorem}
\newtheorem{lemma}{Lemma}

\theoremstyle{definition}

\begin{document}
\title[Robbins-Monro conditions for persistent exploration learning strategies]{Robbins-Monro conditions for persistent exploration learning strategies}

\author{Dmitry B. Rokhlin}

\address{Institute of Mathematics, Mechanics and Computer Sciences,
              Southern Federal University,
Mil'chakova str., 8a, 344090, Rostov-on-Don, Russia}
\email[Dmitry B. Rokhlin]{rokhlin@math.rsu.ru}


\begin{abstract}
We formulate simple assumptions, implying the Rob\-bins-Mon\-ro conditions for the $Q$-learning algorithm with the local learning rate, depending on the number of visits of a particular state-action pair (local clock) and the number of iteration (global clock). It is assumed that the Markov decision process is communicating and the learning policy ensures the persistent exploration. The restrictions are imposed on the functional dependence of the learning rate on the local and global clocks. The result partially confirms the conjecture of Bradkte (1994).
 
\end{abstract}
\subjclass[2010]{93E35, 62L20}

\keywords{Robbins-Mobro conditions, reinforcement learning, learning rate, learning strategy, persistent expolration, communicating Markov decision processes}

\maketitle

\section{Problem description}
\label{sec:1}
Reinforcement learning is aimed at the solution of the Markov decision problems without the exact knowledge of an underlying model.  In this paper we address only the case of finite state-action Markov decision processes (MDP). Moreover, for concreteness we discuss only the discounted optimality criterion and the $Q$-learning algorithm. However, this is not essential since we consider only the Robbins-Monro conditions for the learning rates, and not the convergence of the algorithms. So, the result is applicable to other reinforcement learning algorithms, based on asynchronous stochastic approximation.

The $Q$-learning can be regarded as an asynchronous version of the classical value iteration algorithm for the $Q$-function. Recall that a $Q$-function $Q(x,a)$ is the optimal gain for fixed initial state $x$ and initial action $a$. The $Q$-learning algorithm updates the current approximation $Q_t$ to $Q$ along a trajectory $(x_t,a_t)$ of states $x_t$ and actions $a_t$, generated by selected learning (or exploration) strategy.

A learning strategy is a sequence of probability distributions $\pi_t(a)$ on the action set $A$ (we assume that $A$ does not depend on $x$). As e.g. in \cite{singh2000}, we distinguish between persistent exploration and decaying exploration learning strategies. Persistent exploration (in contrast to the decaying one) means the existence of a uniform lower bound of the form $\pi_t(a)\ge c>0$.

Besides the learning strategy, a particular instance of the $Q$-learning algorithm is determined by a learning rate $\gamma_t(x,a)$ which controls the influence of the new information on the update rule. Usually the learning rate is of the form
\begin{equation} \label{1.0}
\gamma_t(x,a)=\alpha_t I_{\{x_t=x,a_t=a\}}.
\end{equation}
The sequence $(\alpha_t)$ will be also called a learning rate.
The standard results assert the pointwise convergence $Q_t\to Q$ with probability 1 under the Robbins-Monro conditions (see Theorem \ref{th:1}):
\begin{equation} \label{1.1}
\sum_{t=0}^\infty \gamma_t=\infty,\quad \sum_{t=0}^\infty\gamma_t^2<\infty.
\end{equation}

Clearly, it is required that each state-action pair $(x,a)$ is visited infinitely often. Assuming that this property is satisfied, it is easy to construct a sequence $(\alpha_t)$ depending on a ``local clock'' and verifying (\ref{1.1}). By a local clock we mean the number of visits of a particular point $(x,a)$ by the sequence $(x_t,a_t)$. Indeed, consider a function $\varphi:\mathbb Z_+\mapsto(0,\infty)$ satisfying the Robbins-Monro conditions, that is,
$$\sum_{t=1}^\infty \frac{1}{\varphi(t)}=\infty,\qquad \sum_{t=1}^\infty \frac{1}{\varphi^2(t)}<\infty.$$
Put $\alpha_t=1/\varphi(n_t(x,a))$, where 
\begin{equation} \label{1.3}
n_t(x,a)=\sum_{k=0}^t I_{\{x_k=x,a_k=a\}}
\end{equation}
is the number of visits of $(x,a)$ by the sequence $(x_k,a_k)_{k=0}^t$, and denote by $t_j(x,a)$ the time of $j$-th visit, $j\ge 1$. Then $n_{t_j}(x,a)=j$ and
$$\sum_{t=0}^\infty \gamma_t=\sum_{j=1}^\infty\alpha_{t_j}=\sum_{j=1}^\infty\frac{1}{\varphi(n_{t_j}(x,a))}= \sum_{j=1}^\infty\frac{1}{\varphi(j)}=\infty.$$
Similarly,
$$\sum_{t=0}^\infty \gamma_t=\sum_{j=1}^\infty\frac{1}{\varphi^2(j)}<\infty.$$

If the learning rate $\alpha_t$ explicitly depends on the ``global clock'', that is, the iteration number $t$, then the situation becomes more difficult. Let $\alpha_t$ be a deterministic sequence. In his PhD thesis Bradtke (\cite{bradtke1994}, see also \cite{bradtke1996}) in somewhat different situation, involving function approximation, conjectured that if $(\alpha_t)$ satisfies the Robbins-Monro conditions:
$$\sum_{t=0}^\infty \alpha_t=\infty,\quad \sum_{t=0}^\infty\alpha_t^2<\infty,$$
then the same is true for $\gamma_t$. In \cite{szepesvari1996} it was mentioned that this conjecture is true if the inter-arrival times $t_{j+1}-t_j$ have a common upper bound or, more specifically, are eventually exponentially distributed with common parameters. However, these conditions are difficult to verify and they depend on the learning strategy.

In this note we show that the Bradtke conjecture holds true for persistent exploration learning strategies. This assertion follows from the main result: Theorem \ref{th:2}. 

\section{Markov decision processes and $Q$-learning} 
\label{sec:2}
Let $X$ and $A$ be finite state and action spaces. Consider the canonical space $\Omega=(X\times A)^\infty$ with the $\sigma$-algebra $\mathscr F$ generated by projections 
$$(x_0,a_0,x_1,a_1,\dots)\mapsto (x_t,a_t).$$
Denote by $\mathscr F_t=\sigma(x_0,a_0,\dots,s_t,a_t)$ the natural filtration of the coordinate process. The probabilistic structure of the process $(x_t,a_t)$ is determined by a fixed transition kernel $q(y|x,a)$:
$$ \sum_{y\in X} q(y|x,a)=1,\qquad q(y|x,a)\ge 0$$
and a control (or learning) strategy, which is a sequence $\pi=(\pi_t)$ of probability distributions on the action set $A$. These objects uniquely determine a unique probability measure $\mathsf P_{z,\pi}$ on $\Omega$ such that
\begin{align*}
&\mathsf P_{z,\pi}(x_{t+1}=y|\mathscr F_t,a_t)=q(y|x_t,a_t),\quad \mathsf P_{z,\pi}(a_t=a|\mathscr F_{t-1},x_t)=\pi_t(a),\\
&\mathsf P_{z,\pi}(x_0=z)=1
\end{align*}
(see, e.g., \cite{lerma1996}). Note, that $\pi_t(a)$ is $\sigma(\mathscr F_{t-1},x_t)$-measurable.

Given a reward function $r(x,a,y)$ and a discounting factor $\beta\in [0,1)$, the total discounted gain is defined by the value function
$$ V(z)=\sup_{\pi}\mathsf E_{z,\pi}\sum_{t=0}^\infty \beta^t r(x_t,a_t,x_{t+1}),$$
where $\mathsf E_{z,\pi}$ is the expectation with respect to $\mathsf P_{z,\pi}$. As is well known, this function is a unique solution of the Bellman (or dynamic programming) equation:
$$ V(x)=\max_{a\in A}\sum_{y\in X} q(y|x,a)(r(x,a,y)+\beta V(y)).$$
The $Q$-function is the total discounted gain for fixed initial state and initial action:
$$ Q(x,a)=\sum_{y\in X} q(y|x,a)(r(x,a,y)+\beta V(y)).$$
This function is a unique solution of the equation
$$ Q(x,a)=\sum_{y\in X} q(y|x,a)(r(x,a,y)+\beta \max_{a\in A} Q(y,a)).$$

The $Q$-learning algorithm proposed in \cite{watkins1989} recursively defines the sequence $Q_t$:
\begin{align*}
Q_{t+1}(x,a)&=(1-\alpha_t I_{\{x_t=x,a_t=a\}}) Q_t(x,a)\label{1.4}\\
&+\alpha_t I_{\{x_t=x,a_t=a\}}(r(x_t,a_t,x_{t+1})+\beta\max_{a'\in A}Q_t(x_{t+1},a'))\nonumber 
\end{align*}
for a strictly positive $\mathscr F_t$-measurable random variables $\alpha_t$ and an arbitrary initial guess $Q_0(x,a)$.

Let us recall a basic result on the convergence of $Q_t$ to $Q$ with  probability $1$: see \cite{jaakkola1994,tsitsiklis1994}.
\begin{theorem} \label{th:1}
Assume that the learning rate $\alpha_t$ satisfies the Robbins-Monro conditions
\begin{equation} \label{2.1}
\sum_{t=0}^\infty \alpha_t I_{\{x_t=x,a_t=a\}}=\infty,\quad \sum_{t=0}^\infty \alpha_t^2 I_{\{x_t=x,a_t=a\}}<\infty\quad \mathsf P_{z,\pi}\textrm{-a.s.}
\end{equation}
for all $(x,a)\in X\times A$. Then
$$ \lim_{t\to\infty} Q_t(x,a)=Q(x,a)\quad\mathsf P_{z,\pi}\textrm{-a.s.}$$
\end{theorem}

In this paper we study only conditions (\ref{2.1}) and not the proof of Theorem \ref{th:1}. Under the assumption that each pair $(x,a)$ is visited infinitely often, one simple construction of the learning rate $\alpha_t$, depending on the local clock (\ref{1.3}) and satisfying (\ref{2.1}), was given is Section \ref{sec:1}. In the sequel we solely consider another version of a local clock, defined as the number of visits of a particular state $x$:
\begin{equation} \label{2.2}
 N_t(x)=\sum_{k=0}^t I_{\{x_k=x\}}.
\end{equation}
 
Assume that all states are visited infinitely often $\mathsf P_{z,\pi}$-a.s., the learning strategy satisfies the lower bound $\pi_t(a)\ge c(N_t)>0$, the learning rate is of the form $\alpha_t=1/\varphi(N_t)$ and
\begin{equation} \label{2.3}
 \sum_{t=1}^\infty\frac{c(t)}{\varphi(t)}=\infty,\qquad \sum_{t=1}^\infty\frac{1}{\varphi^2(t)}<\infty,
 \end{equation}
then the Robbins-Monro conditions (\ref{2.1}) are satisfied. 

Indeed, by the conditional Borel-Cantelli lemma \cite{meyer1972} (Chapter 1, Theorem 21), the first condition (\ref{2.1}) is satisfied if and only if
\begin{align} \label{2.4}
&\sum_{t=0}^\infty \mathsf E_{z,\pi}(\alpha_t I_{\{x_t=x,a_t=a\}}|\mathscr F_{t-1},x_t)=
\sum_{t=0}^\infty \frac{1}{\varphi(N_t)} I_{\{x_t=x\}}\mathsf E_{z,\pi}(I_{\{a_t=a\}}|\mathscr F_{t-1},x_t)\nonumber\\
=&\sum_{t=0}^\infty \frac{1}{\varphi(N_t)} I_{\{x_t=x\}}\pi_t(a)\ge\sum_{t=0}^\infty\frac{1}{\varphi(N_t)} I_{\{x_t=x\}}c(N_t)=
\sum_{j=1}^\infty \frac{c(j)}{\varphi(j)}=\infty 
\end{align}
$\mathsf P_{z,\pi}\textrm{-a.s.}$
For the second condition (\ref{2.1}) the argumentation is even easier: 
$$\sum_{t=0}^\infty \alpha_t^2 I_{\{x_t=x,a_t=a\}}\le \sum_{t=0}^\infty \frac{1}{\varphi^2(N_t)} I_{\{x_t=x\}}=\sum_{j=1}^\infty \frac{1}{\varphi^2(j)}<\infty \quad \mathsf P_{z,\pi}\textrm{-a.s.}$$

Note, that the decaying exploration is allowed, but the learning strategy should ensure infinitely many visits of every state and the lower bounds $c(t)$ should be consistent with learning rate: see the first condition (\ref{2.3}).

In the next section we allow an explicit dependence of $\alpha_t$ on the global clock $t$, but consider only persistent exploration learning strategies. Two main examples of persistent exploration learning strategies are 
\begin{itemize}
\item the Boltzmann exploration:
$$ \pi_t(a)=\frac{\exp(Q_t(x_t,a)/\tau)}{\sum_{a'}\exp(Q_t(x_t,a')/\tau)},\quad \tau>0.$$
The required inequality $\pi_t(a)\ge c>0$ follows from the boundedness of the sequence $(Q_t)$: see \cite{gosavi2006} for a simple proof.
\item $\varepsilon$-greedy exploration which takes a ``greedy'' action $a_t\in\arg\max Q_t(x_t,a_t)$ with probability $1-\varepsilon$ and a random action with probability $\varepsilon$.
\end{itemize}

\section{Robbins-Monro conditions for persistent exploration learning strategies}
\label{sec:3}
A distribution $a\mapsto g(a|x)$ on $A$, defined for all $x\in X$, is called a stationary randomized strategy. If $g(b(x)|x)=1$  for some function $b:X\mapsto A$, then the strategy is called deterministic. Such strategy can be identified with the function $b$. A stationary randomized strategy is called completely mixed if  $g(a|x)>0$ for all $x\in X$, $a\in A$. Any stationary randomized strategy $g$ induces a Markov chain with the transition matrix 
$$ P(g)(x,y)=\sum_{a\in A} q(y|x,a) g(a|x).$$

An MDP is called communicating (see \cite{bather1973,filar1988,kallenberg2002}), if for any $x, y \in X$ there exists a stationary deterministic strategy $g$ such that $y$ is accessible from $x$ in the Markov chain $P(g)$. In other words, there exists $n\in\mathbb N$, depending on $x, y$, such that $P^n(g)(x,y)>0$. We will use the fact that an MDP is communicating if and only $P(g)$ is irreducible for every completely mixed stationary randomized strategy: see \cite[Theorem 2.1]{filar1988}. 

Define  the completely mixed strategy $\overline g(a|x)=1/|A|$, where $|A|$ is the cardinality of $A$. Let us recall (see \cite[Lemma 7.3(i)]{behrends2000}) that a Markov chain $P(\overline g)$ is irreducible if and only if there exist $n\in\mathbb N$ such that the matrix $\sum_{j=1}^n P^j(\overline g)$ is strictly positive. Let $\delta>0$ be the minimal element of this matrix. Then
\begin{equation} \label{3.1}
\sum_{j=1}^n P^j(\overline g)(x,y)\ge\delta.
\end{equation}
\begin{lemma} \label{lem:1}
Assume that an MDP is communicating and the learning strategy $\pi$ ensures the persistent exploration: $\pi_t(a)\ge c>0$. Then for any function $f:X\mapsto [0,\infty)$ we have
\begin{equation} \label{3.2}
\sum_{j=1}^n \mathsf E_{z,\pi}(f(x_{t+j+1})|\mathscr F_t)\ge c^n |A|^n\delta\sum_{y\in X} f(y),
\end{equation}
where the constants $n$, $\delta$ satisfy (\ref{3.1}).
\end{lemma}
\begin{proof} Put
$$ P^n(g) f(x)=\sum_{y\in X} P^n(g)(x,y) f(y),\quad n\ge 1.$$
Let $k\ge 2$, $f\ge 0$. Then
\begin{align*}
&\mathsf E_{z,\pi}(f(x_{t+k})|\mathscr F_t)=\sum_{x} f(x)\mathsf P_{z,\pi}(x_{t+k}=x|\mathscr F_t)\\
&=\sum_{x} f(x)\mathsf E_{z,\pi}(\mathsf P_{z,\pi}(x_{t+k}=x|\mathscr F_{t+k-1})|\mathscr F_t)\\
&=\sum_{x} f(x)\mathsf E_{z,\pi}(q(x|x_{t+k-1},a_{t+k-1})|\mathscr F_t)\\
&=\sum_{x} f(x)\mathsf E_{z,\pi}(\mathsf E_{z,\pi}(q(x|x_{t+k-1},a_{t+k-1})|\mathscr F_{t+k-2},x_{t+k-1})|\mathscr F_t)\\
&=\sum_{x} f(x)\mathsf E_{z,\pi}\left(\sum_a q(x|x_{t+k-1},a)\pi_{t+k-1}(a)|\mathscr F_t\right)\\
&\ge c \sum_{x} f(x)\mathsf E_{z,\pi}\left(\sum_a q(x|x_{t+k-1},a)|\mathscr F_t\right)\\
&= c |A| \sum_{x} f(x)\mathsf E_{z,\pi}\left(P(\overline g)(x_{t+k-1},x)|\mathscr F_t\right)\\
&=c|A|\mathsf E_{z,\pi}\left(P(\overline g) f(x_{t+k-1})|\mathscr F_t\right)\ge c^{k-1}|A|^{k-1}\mathsf E_{z,\pi}\left(P^{k-1}(\overline g) f(x_{t+1})|\mathscr F_t\right).
\end{align*}

It follows that
\begin{align*}
&\sum_{j=1}^n\mathsf E_{z,\pi}(f(x_{t+j+1})|\mathscr F_t)\ge \sum_{j=1}^n c^j |A|^j\mathsf E_{z,\pi}\left(P^j(\overline g) f(x_{t+1})|\mathscr F_t\right)\\
&=\sum_{j=1}^n c^j |A|^j\sum_{z}P^j(\overline g) f(z)q(z|x_t,a_t)\ge c^n |A|^n\sum_{z}\sum_{j=1}^n P^j(\overline g) f(z)q(z|x_t,a_t)\\
&\ge c^n |A|^n \min_z \sum_{j=1}^n P^j(\overline g) f(z)
=c^n  |A|^n\min_z \sum_y\sum_{j=1}^n P^j(\overline g)(z,y) f(y)\\
&\ge c^n |A|^n\delta\sum_{y} f(y),
\end{align*}
where we used the fact that $c\le 1/|A|$. 
\end{proof}

Under the assumptions of Lemma \ref{lem:1} every state $x\in X$ is visited infinitely often. It is even possible to give a lower bound for the growth rate of the local clock $N_t$. Namely, we claim that
\begin{equation} \label{3.2A}
 \liminf_{t\to\infty}\frac{N_t(x)}{t}\ge \frac{c^n|A|^n\delta}{n}\qquad \mathsf P_{z,\pi}\textrm{-a.s.}
\end{equation} 

To prove (\ref{3.2A}) let us represent $N_{kn+1}$, $k\ge 1$ in the form
$$ N_{kn+1}(x)=I_{\{x_0=x\}}+I_{\{x_1=x\}}+\sum_{j=1}^k\xi_j,\quad \xi_j=\sum_{l=(j-1)n+2}^{jn+1} I_{\{x_l=x\}}.$$
Furthermore, consider the Doob decomposition
$$\sum_{j=1}^k\xi_j=A_k+M_k,\quad k\ge 1$$
with respect to the filtration $\overline{\mathscr F}_k=\mathscr F_{kn}$, $k\ge 0$. Here $(A_k)$ is a predictable process (compensator):
$$ A_k=\sum_{j=1}^k\mathsf E_{z,\pi}(\xi_j|\overline{\mathscr F}_{j-1})$$
and $(M_k)$ is a martingale. By Lemma \ref{lem:1} we have
\begin{align*}
\mathsf E_{z,\pi}(\xi_j|\overline{\mathscr F}_{j-1})= \sum_{l=(j-1)n+2}^{jn+1} \mathsf E_{z,\pi}\left(I_{\{x_l=x\}}|\mathscr F_{(j-1)n}\right)\\
=\sum_{r=1}^{n}\mathsf E_{z,\pi}\left(I_{\{x_{(j-1)n+r+1}=x\}}|\mathscr F_{(j-1)n}\right)\ge c^n|A|^n\delta.
\end{align*}

It follows that $A_k\ge c^n|A|^n\delta k$. Furthermore,
$$ \frac{M_k}{k}\to 0,\quad k\to\infty\qquad \mathsf P_{z,\pi}\textrm{-a.s.}$$
by the law of large numbers for martingales: \cite[Chapter 7, \S3, Corollary 2]{shiryaev1996}. Thus,
\begin{equation} \label{3.2B}
\liminf_{k\to\infty}\frac{N_{kn+1}(x)}{k}\ge c^n|A|^n\delta\qquad \mathsf P_{z,\pi}\textrm{-a.s.}
\end{equation}
For any $t\in\mathbb N$ there exists a unique $k\in\mathbb N$ such that $t\in[kn,(k+1)n)$. So, the inequality (\ref{3.2A}) easily follows from (\ref{3.2B}):
\begin{align*}
\liminf_{t\to\infty}\frac{N_t(x)}{t}&\ge\liminf_{k\to\infty}\frac{N_{kn}(x)}{(k+1)n}=\liminf_{k\to\infty}\frac{N_{(k+2)n}(x)}{(k+3)n}\\
&\ge \liminf_{k\to\infty}\frac{N_{(k+1)n+1}(x)}{k(1+3/k)n}\ge\frac{c^n|A|^n\delta}{n} \qquad \mathsf P_{z,\pi}\textrm{-a.s.}
\end{align*}

In Theorem \ref{th:2}, which is the main result of this note, the learning rate will be determined by a function $\varphi:\mathbb N\times \mathbb N\mapsto (0,\infty)$. Assume that
\begin{itemize}
\item[(i)] the functions $t\mapsto\varphi(t,j)$, $j\mapsto\varphi(t,j)$ are non-decreasing;
\item[(ii)] the function $\varphi$ satisfies the Robbins-Monro conditions on the diagonal:
\begin{equation} \label{3.3A}
\sum_{t=1}^\infty\frac{1}{\varphi(t,t)}=\infty,\qquad \sum_{t=1}^\infty\frac{1}{\varphi^2(t,t)}<\infty.
\end{equation}
\end{itemize}

\begin{theorem}\label{th:2}
Assume that the MDP is communicating and $\varphi$ satisfies conditions (i), (ii) above. Then  the learning rate $\alpha_t=\varphi(t,N_t)$ satisfies the Robbins-Monro conditions (\ref{2.1}) for a persistent exploration learning strategy: $\pi_t(a)\ge c>0$.
\end{theorem}
\begin{proof} (a) Let us check the first property (\ref{2.1}). We will use the notation (\ref{1.0}). By the conditional Borel-Cantelli lemma the series 
$$\gamma_0+\gamma_1+\sum_{j=1}^k\zeta_j,\qquad \zeta_j=\sum_{l=(j-1)n+2}^{jn+1} \gamma_l$$
diverges $\mathsf P_{z,\pi}$-a.s. if and only if
\begin{equation} \label{3.3}
\sum_{j=1}^\infty \mathsf E_{z,\pi}(\zeta_j|\overline{\mathscr F}_{j-1})=\infty\qquad \mathsf P_{z,\pi}\textrm{-a.s.},
\end{equation}
where $\overline{\mathscr F}_j=\mathscr F_{jn}$. Using the monotonicity properties of $\varphi$ and the inequality (\ref{3.2}), we get
\begin{align*}
\mathsf E_{z,\pi}(\zeta_j|\overline{\mathscr F}_{j-1})&=\sum_{l=(j-1)n+2}^{jn+1} \mathsf E_{z,\pi}\left(\gamma_l|\mathscr F_{(j-1)n}\right)\\
&=\sum_{l=(j-1)n+2}^{jn+1} \mathsf E_{z,\pi}\left(\frac{1}{\varphi(l,N_l)}I_{\{x_l=x\}}\mathsf E_{z,\pi}(I_{\{a_l=a\}}|\mathscr F_{l-1},x_l)|\mathscr F_{(j-1)n}\right)\\
&\ge c \sum_{l=(j-1)n+2}^{jn+1}\mathsf E_{z,\pi}\left(\frac{1}{\varphi(l,l)}I_{\{x_l=x\}}|\mathscr F_{(j-1)n}\right)\\
&\ge \frac{c}{\varphi(jn+1,jn+1)}\sum_{l=(j-1)n+2}^{jn+1}\mathsf E_{z,\pi}\left(I_{\{x_l=x\}}|\mathscr F_{(j-1)n}\right)\\
&\ge \frac{c}{\varphi(jn+1,jn+1)}\sum_{r=1}^{n} \mathsf E_{z,\pi}\left(I_{\{x_{(j-1)n+r+1}=x\}}|\mathscr F_{(j-1)n}\right)\\
&\ge \frac{c^{n+1}|A|^n\delta}{\varphi(jn+1,jn+1)}.
\end{align*}
So, to proof (\ref{3.3}), and hence the first relation (\ref{2.1}), it is enough to show that
$$\sum_{j=1}^\infty \frac{1}{\varphi(jn+1,jn+1)}=\infty.$$
But it is clear, since $\varphi(jn+1,jn+1)\le\varphi(jn+k,jn+k)$, $k=1,\dots,n$ and
$$\infty=\sum_{t=1}^\infty\frac{1}{\varphi(t,t)}\le \sum_{j=0}^\infty\frac{n}{\varphi(jn+1,jn+1)}.$$

(b) Denote by $\tau_j(x)$ the time of $j$-th visit, $j\ge 1$ of the point $x$ by the sequence $(x_t)$. Then $N_{\tau_j}(x)=j$
and 
\begin{align*}
\sum_{t=0}^\infty \frac{1}{\varphi^2(t,N_t)} I_{\{x_t=x,a_t=a\}}&\le \sum_{t=0}^\infty \frac{1}{\varphi^2(t,N_t)} I_{\{x_t=x\}}\\
&=\sum_{j=1}^\infty \frac{1}{\varphi^2(\tau_j(x),j)}\le \sum_{j=1}^\infty \frac{1}{\varphi^2(j,j)}
\end{align*}
since $\tau_j(x)\ge j$ and the function $t\mapsto\varphi(t,j)$ is non-decreasing. Thus, the second condition (\ref{2.1}) is implied by the second condition (\ref{3.3A}). 
\end{proof}

For instance, the learning rates
$$ \varphi(t,N_t)=\frac{a_1}{(b_1+t)^\alpha} \frac{a_2}{(b_2+N_t)^\beta},\quad \alpha+\beta\in (1/2,1],\quad a_i,b_i,\alpha,\beta>0,$$
$$ \varphi(t,N_t)=\frac{a_1}{(b_1+\ln t)^\alpha} \frac{a_2}{(b_2+N_t)^\beta},\quad \alpha\in (1/2,1],\quad \beta\in [1/2,1],\quad a_i,b_i>0$$
satisfy the conditions of Theorem \ref{th:2}.

For the learning rate depending only on the global clock: 
$$\gamma_t=\frac{1}{\varphi(t)}I_{\{x_t=x,a_t=a\}} $$
Theorem \ref{th:2} partially confirms the mentioned conjecture of Bradtke:
$$  \sum_{t=1}^\infty \frac{1}{\varphi(t)}=\infty\quad\Longrightarrow\quad \sum_{t=1}^\infty \frac{1}{\varphi(t)}I_{\{x_t=x,a_t=a\}}=\infty $$
for finite state-action communicating MDP, persistent exploration learning strategies and non-decreasing functions $\varphi$. 

It would be interesting to investigate the case of decaying exploration learning strategies. It is clear that the Robbins-Monro conditions (\ref{2.1}) can be ensured only by joint conditions on the learning rate and the randomized learning strategy $(\pi_t)$. A simple illustration was given by (\ref{2.4}).

\subsubsection*{Acknowledgments.} The research is supported by the Russian Science Foundation, project 17-19-01038. 

%
%

\end{document}